\documentclass{article}
\usepackage{graphicx} 

\title{Path Learning with\\Trajectory Advantage Regression}
\author{
    Kohei Miyaguchi\\
    \texttt{koheimiyaguchi@gmail.com}
}

\input{definitions}

\begin{document}

\maketitle
\begin{abstract}
    In this paper,
    we propose trajectory advantage regression,
    a method of offline path learning and path attribution based on reinforcement learning.
    The proposed method can be used to solve path optimization problems
    while algorithmically only solving a regression problem.
\end{abstract}


\section{Introduction}

We are concerned with the problem of path learning~(PL) in an offline fashion.
The goal of PL is to find the path $\psi$ maximizing the yield $J(\psi)$,
whereas
the feasible set of paths $\Psi$
and
the shape of the yield function $J:\Psi\to \doubleR$
are both (partially or entirely) unknown,
hence to be estimated from fixed observational data collected in advance (i.e., the \emph{offline setting}).

To address this problem, we propose to frame the offline PL problem
as a special sub-problem of the offline reinforcement learning~(RL) and derived a novel algorithm to solve it.
This algorithm allows us to find the optimal path in $\Psi$ efficiently
and also gives a new path-scoring method useful for explaining the (sub-) optimality of paths in terms of the path elements.

The rest of the paper is organized as follows.
We start with introducing preliminary facts and formulations in \cref{sec:preliminary}.
Then, we show the reducibility of PL to RL in \cref{sec:pl2rl}, which is the key to our method presented in \cref{sec:method}.
Finally, we conclude the paper discussion the related work in \cref{sec:related_work} and summarizing the findings and future directions in \cref{sec:conclusion}.

\section{Preliminaries}
\label{sec:preliminary}
$\Delta(\Xcal)$ denotes the set of probability measures on $\Xcal$.
$\delta_x$ denotes the Dirac's point measure at $x$.
$\Xcal^*\coloneqq \bigcup_{T\ge 0}\Xcal^T$ is the set of sequences $(x_1,\ldots,x_T)\in\Xcal^T$ with variable length $T\ge 0$.
In particular, let $\emptyset\in\Xcal^*$ denote the empty sequence with length zero.
If a mapping $f:\Xcal\to\Delta(\Ycal)$ is deterministic,
i.e., there exists 
$\fhat:\Xcal\to\Ycal$ such that $f(x)=\delta_{\fhat(x)}$ for all $x\in\Xcal$,
we may slightly abuse the notation and identify $f$ as $\fhat$.

\paragraph{Offline path learning (PL).}
Let $\Acal$ be a finite set of \emph{actions} with which we assemble a \emph{path},\footnote{
    $a\in\Acal$ can be seen as an edge of a graph.
    However, we refer to it as an action to establish connection with RL later.
}
or a sequence of actions, $\psi=a^T=(a_1,a_2,\ldots,a_T)\in\Acal^*$.
Without loss of generality, we assume every path $a^T$ of interest is ended with the special action $a_T=\bot\in\Acal$ that marks the end of the sequence.
We refer to such $\bot$-terminated sequences as \emph{complete sequences},
and to their prefixes as \emph{proper sequences}.
Let $J_\Psi(\psi)\in[0,1]$ be the \emph{yield} of each path $\psi$
and $\Psi\subset \Acal^*$ be a finite feasible set of complete sequences,
where $J_\Psi(\psi)=0$ if $\psi\not\in\Psi$.

The goal of the offline path learning (PL) is 
to maximize the yield $J_\Psi(\psi)$
given noisy path-yield pairs $Z\coloneqq \cbr{(\psi_i,y_i)}_{i=1}^n$,
while not knowing the feasible set $\Psi$, the yield function $J_\Psi(\cdot)$,
nor the distribution of $(\psi_i, y_i)$ other than $\doubleE\sbr{y_i|\psi_i}=J_\Psi(\psi_i)$.
For simplicity, we assume each pair $(\psi_i, y_i)$ is independent of one another and denote the underlying distributions of $\psi_i$ and $y_i|\psi_i$ by
$P_\Psi\in\Delta(\Psi)$ and $P_Y:\Psi\to \Delta([0, 1])$.
Thus, an instance of offline PL is identified by the tuple $\Pcal\equiv (\Acal, \Psi, P_\Psi, P_Y)$.

\paragraph{MDP and RL.}
A Markov decision process~(MDP)~\citep{puterman2014markov,sutton1998reinforcement} $\Mcal\equiv (\Scal, \Acal, P_1, P_r, P_+)$ consists of 
the state space $\Scal$,
the action space $\Acal$ (intentionally using the same action set as PL),
the initial state distribution $P_1\in\Delta(\Scal)$,
the reward distribution $P_r:\Scal\times \Acal\to \Delta([0,1])$, and
the transition kernel $P_+:\Scal\times \Acal\to \Delta(\Scal)$.
With a policy $\pi:\Scal\to\Delta(\Acal)$,
an MDP $\Mcal$ generates \emph{episodes} $\xi=\rbr{(s_t,a_t,r_t)}_{t=1}^\infty \sim P_\xi(\Mcal,\pi)$,
where
$s_1\sim P_1$, $a_t\sim \pi(s_t)$, 
$s_{t+1}\sim P_+(s_t,a_t)$ and
$r_t\sim P_r(s_t,a_t)$
for $t\ge 1$.

The goal of the reinforcement learning~(RL) is to find a policy $\pi$
that maximizes the total expected reward (or the \emph{policy value})
\begin{align}
    J_\Pi(\pi)\coloneqq \doubleE^\pi\sbr{\sum_{t=1}^\infty r_t},
    \label{eq:policy_value}
\end{align}
where the expectation $\doubleE^{\pi}$ is taken with respect to the episode $\rbr{(s_t,a_t,r_t)}_{t=1}^T\sim P_\xi(\Mcal,\pi)$.
Here, we assume the expectation in \cref{eq:policy_value} is well-defined.
An instance of the reinforcement learning~(RL) is identified
with the MDP $\Mcal$.

\paragraph{Offline RL.}
The offline RL~\citep{levine2020offline} is a special class of RL,
where we cannot access $\Mcal$ except through
a fixed dataset $\Xi\coloneqq \cbr{(s_i,a_i,r_i,s'_i)}_{i=1}^n$,
where $(s_i,a_i)\sim \mu\in \Delta(\Scal\times \Acal)$ is a marginal state-action distribution
and $r_i\sim P_r(s_i,a_i)$, $s'_i\sim P_+(s_i,a_i)$.
We identify an instance of the offline RL by the tuple $\Rcal\equiv (\Scal, \Acal, P_1, P_r, P_+, \mu)$.

\paragraph{Value functions and optimal policy.}
The state value function~(or V-function) $V^\pi:\Scal\to \doubleR$
of policy $\pi\in\Pi\coloneqq \Delta(\Acal)^\Scal$
is defined as
the conditional total expected reward
\begin{align*}
    V^\pi(s)\coloneqq \doubleE^\pi\sbr{\sum_{t=1}^\infty r_t \middle | s_1=s},
\end{align*}
whereas the action value function~(or Q-function) is given by
\begin{align*}
    Q^\pi(s,a)\coloneqq \Tcal V^\pi (s,a),
\end{align*}
with $\Tcal:\doubleR^{\Scal}\to\doubleR^{\Scal\times\Acal}$ being the transition operator such that
$$\Tcal f(s,a)=\doubleE_{r\sim P_r(s,a),s'\sim P_+(s,a)}\sbr{r+f(s')}.$$
We refer to the pointwise maximums of V- and Q- functions with respect to policy $\pi\in\Pi$ as the \emph{optimal V- and Q- functions},
$V^*(s)\coloneqq \max_{\pi\in\Pi}V^\pi(s)$, $Q^*(s,a)\coloneqq max_{\pi\in\Pi}Q^\pi(s,a)$.
The \emph{(optimal) advantage function} is then given by $A^*(s,a)=Q^*(s,a)-V^*(s)~(\le 0)$,
quantifying the goodness of action $a$ at state $s$ compared to the best policy in $\Pi$.

It has been known that there exists a deterministic optimal policy $\pi^*:\Scal\to \Delta(\Acal)$,
such that $V^{\pi^*}(s)=V^*(s)$ and $Q^{\pi^*}(s,a)=Q^*(s,a)$~(Theorem~7.1.9 in \cite{puterman2014markov}).

\paragraph{Linear programming for the optimal V-function.}
The optimal V-function can be characterized as the solution of a linear programming called V-LP~(\cite{puterman2014markov}, Section 7.2.7),
\begin{subequations}
\label{eq:vlp}
\begin{align}
    &\min \doubleE_{s\sim P_0}\sbr{V(s)}
    \\
    &\quad \text{s.t.}\quad \Tcal V(s,a)\le V(s),\quad s\in\Scal,a\in\Acal,
    \\
    &\quad \phantom{\text{s.t.}}\quad V(s)\ge 0,\quad s\in\Scal.
    \label{eq:vlp:constraint}
\end{align}
\end{subequations}
Here, $P_0\in\Delta(\Scal)$ is a fixed state distribution covering the entire state space $\mathcal S$.
\cref{eq:vlp} can be approximated
with the penalty method,
turned into the minimization of 
\begin{align}
    \Lcal(V)&\coloneqq \doubleE_{s\sim P_0}\sbr{V(s)} + \lambda \doubleE_{(s,a)\sim P_1}\sbr{\{\Tcal V(s,a) - V(s)\}_+^2}, &V\ge 0,
    \label{eq:penalized_vlp}
\end{align}
where $\lambda\ge 0$ is the penalty coefficient,
$P_1\in\Delta(\Scal\times \Acal)$ is a fixed distribution supported on the entire state-action space,
and $\cbrinline{x}_+\coloneqq \max(0, x)$.


\section{PL-to-RL Reduction}
\label{sec:pl2rl}
In this section, we show that the offline PL can be reduced to the offline RL
by constructing a mapping from offline-PL instances $\Pcal$
to offline-RL instances $\Rcal$.

Take the state space as the path space, $\Scal=\Acal^*$.
The initial state is then naturally given as the empty sequence, $P_1=P_1^\oplus\coloneqq\delta_{\emptyset}$,
and the state transition is given by appending actions one by one,
\begin{align*}
    P_+^\oplus(s,a)\coloneqq
    s\oplus a.
\end{align*}
Moreover, take the following reward distribution associated with the noisy yield distribution $P_Y$,
\begin{align*}
    P_r^Y(s,a)&\coloneqq
    \left\{\begin{array}{cl}
            P_Y(s) & (s\in\Psi)\\
            \delta_0 & (\text{otherwise})
    \end{array}\right..
\end{align*}%
Finally, take the marginal distribution as
\begin{align}
    \mu^\Psi(s,a)=P_\Psi(s)\cdot U_A(a),
    \label{eq:PLRL_marginal}
\end{align}
where $U_A\in\Delta(\Acal)$ is the uniform distribution over $\Acal$.


Now, let
\begin{align}
    \Rcal_\Pcal\coloneqq (\Acal^*, \Acal,P_1^\oplus,P_r^Y,P_+^\oplus,\mu^\Psi).
    \label{eq:reduction}
\end{align}
be the offline-RL instance constructed based on $\Pcal=(\Acal,\Psi,P_\Psi,P_Y)$ in the aforementioned manner.
Then, an offline dataset of $\Rcal_\Pcal$ is also constructed with that of the original offline PL problem,
$\Xi\coloneqq \{(\psi_i,a_i,y_i,\psi_i\oplus a_i,1)\}_{i=1}^n$, $a_i\sim U_A$.
Moreover,
solutions of $\Rcal_\Pcal$ are equivalent to those of the original instance $\Pcal$.

\begin{proposition}
\label{prop:reduction}
    Let $\pi^*$ be an optimal policy of $\Rcal_\Pcal$.
    Then, $\psi^{\pi^*}$ is an optimal path for offline-PL instance $\Pcal$.
\end{proposition}

\cref{prop:reduction} implies that offline PL can be solved by reduction to offline RL.
Furthermore, the optimal state value $V^*(s)$
is equivalent to the yield in the best-case scenario with the path prefixed by the given sequence $s\in\Acal^*$,
\begin{align}
    V^*(s)
    &=\max_{b\in \Acal^*} J_\Psi(s\oplus b),
    \label{eq:value_interpretation}
\end{align}
and the advantage $A^*(s, a)$ is equivalent to the drawdown of the best-case yield
due to appending action $a$ to $s$,
\begin{align}
    A^*(s,a)
    &=
    \max_{b\in \Acal^*} J_\Psi(s\oplus a\oplus b)
    -\max_{b\in \Acal^*} J_\Psi(s\oplus b)
    \label{eq:advantage_interpretation}
\end{align}
for all $s\not\in\Psi$.
Therefore, the reduction~\eqref{eq:reduction} is useful not just for finding the optimal paths,
but also for gaining insights on the contribution of each action in a path.

\section{Trajectory Advantage Regression (TAR)}
\label{sec:method}

To solve the offline-RL instance given by the PL-to-RL reduction (\cref{prop:reduction}),
we propose the method of \emph{trajectory advantage regression (TAR)}.
TAR is designed to exploit a problem structure in the reduced MDP.
Specifically, we decompose the optimal V-function in terms of the advantage function.
\begin{lemma}[Advantage decomposition]
    \label{lem:advantage_decomposition}
    Consider the reduced MDP~\eqref{eq:reduction}.
    Then,
    \begin{align}
        V^*(a^t)
        &=
        \left\{\begin{array}{cc}
                J^* + \sum_{k=1}^{t} A^*(a^{k-1}, a_k) & (\text{$a^t$ is proper}) \\
                0 & (\text{otherwise})
        \end{array}\right.,
        \label{eq:advantage_decomposition}
    \end{align}
    where
    $J^*\coloneqq \max_{\pi\in \Pi} J_\Pi(\pi)$ is the optimal policy value.
\end{lemma}
\begin{proof}
    Since the improper case is trivial from \cref{eq:value_interpretation},
    we prove the case when $a^t$ is proper.
    Summing \cref{eq:advantage_interpretation}
    with $s=a^{k-1}$ and $a=a_k$ for $1\le k\le t$
    and plugging \cref{eq:value_interpretation},
    we have
    $V^*(a^{t'})= V^*(\emptyset) + \sum_{k=1}^{t'} A^*(a^{k-1}, a_k)$.
    The proof is thus concluded with the identity $V^*(\emptyset)=J^*$.
\end{proof}

Consider the parametrization induced by \cref{lem:advantage_decomposition} of estimates for the optimal V-function,
\begin{align}
    V_\theta(a^t)\coloneqq 
    \left\{\begin{array}{cc}
            c_\theta+\sum_{k=1}^{t} A_\theta(a^{k-1},a_k) & (\text{$a^t$ is proper}) \\
            0 & (\text{otherwise})
    \end{array}\right.,
    \label{eq:tar_parametrization}
\end{align}
where $c_\theta\in\doubleR$ and $A_\theta:\Scal\times\Acal\to \doubleR_{\le 0}$
are estimates for $J^*$ and $A^*$, respectively.

Estimating $V^*$ through \cref{eq:tar_parametrization} is preferred to directly estimating $V^*$ itself for several reasons.
First, from a explainability perspective,
individual advantages $A^*(s,a)$ are more interpretable than the value $V^*(s)$
as it quantifies an effect of choosing action $a$ at state $s$, as discussed in \cref{sec:pl2rl}.
In other words,
\cref{eq:tar_parametrization} can be used to explain the predicted yield $V_\theta(a^T)$
of complete paths $a^T$
in terms of the optimal expected yield $c_\theta$ and the drawdowns $A_\theta(a^{k-1},a_k)$ introduced by the individual actions $a_k$.
Such high explainability also results in the ease of modeling,
since one can design a function approximator directly based on 
the end-user's hypothetical expectations on the form of $A^*(s,a)$.

Second, from a computational perspective,
adopting \cref{eq:tar_parametrization} leads to a significant simplification of the objective function.
Concretely, we propose the following objective function for \cref{eq:tar_parametrization},
\begin{align}
    \Lcal_{\mathrm{TAR}}(V_\theta)
    &\coloneqq 
    \doubleE_{s\sim P_0}\sbr{V_\theta(s)}
    +\frac{\lambda}{2} \doubleE_{(\psi,y)\sim P_{\Psi Y}}\sbr{y-V_\theta(\psi)}^2,&V_\theta \ge 0,
    \label{eq:tar_loss}
\end{align}
which is just a least-squares regression problem
and much easier to optimize than \cref{eq:penalized_vlp}.
We refer to the minimization of $\Lcal_{\mathrm{TAR}}(V_\theta)$ as \emph{trajectory advantage regression~(TAR)}.
TAR is justified by the following theorem.

\begin{theorem}
    \label{thm:justify_tar}
    Suppose $P_\Psi$ is supported on $\Psi$.
    Then,
    there exists a distribution $P_1$ supported on the entire state-action space $\Scal\times \Acal$
    such that 
    \begin{align}
        \Lcal_{\mathrm{TAR}}(V_\theta)= \frac{\lambda \sigma^2}{2} + \Lcal(V_\theta) + \frac{\lambda}{2}\norm{\{V_\theta-V^*\}_+}_{2,P_\Psi}^2
        \label{eq:tar_loss_rewritten}
    \end{align}
    where $\norm{f}_{2,P_\Psi}^2\coloneqq\doubleE_{\psi\sim P_\Psi}\sbrinline{f^2(\psi)}$
    and $\sigma^2\coloneqq \doubleE_{\psi\sim P_\Psi}[\mathrm{Var}_{y\sim P_Y(\psi)}(y)]$.
\end{theorem}
\begin{proof}
    Let $P_1=(\mu^\Psi+\Ptil_1)/2$ with $\Ptil_1$ being an arbitrary distribution supported on $(\Scal\setminus \Psi)\times \Acal$.
    Then, the penalty term in $\Lcal(V_\theta)$ is decomposed as
    \begin{align*}
        \mathop{\doubleE}_{(s,a)\sim P_1}
        \sbr{\cbr{\Tcal V_\theta(s,a)-V_\theta(s)}_+^2}
        &=
        \frac{
            \mathop{\doubleE}_{(s,a)\sim \mu^\Psi}
            +\mathop{\doubleE}_{(s,a)\sim \Ptil_1}
        }2
        \sbr{\cbr{\Tcal V_\theta(s,a)-V_\theta(s)}_+^2}.
    \end{align*}
    Now, if $(s,a)\sim \mu^\Psi$, we have
    \begin{align*}
        \cbr{\Tcal V_\theta(s,a)-V_\theta(s)}_+^2
        &=
        \cbr{J_\Psi(s)+V_\theta(s\oplus a)-V_\theta(s)}_+^2
        \\
        &=
        \cbr{J_\Psi(s)-V_\theta(s)}_+^2.
        &(\because s\in\Psi)
    \end{align*}
    Moreover, if $(s,a)\sim \Ptil_1$,
    we have
    \begin{align*}
        \cbr{\Tcal V_\theta(s,a)-V_\theta(s)}_+^2
        &=
        \cbr{J_\Psi(s)+V_\theta(s\oplus a)-V_\theta(s)}_+^2
        \\
        &=
        \cbr{V_\theta(s\oplus a)-V_\theta(s)}_+^2
        &(\because s\not\in \Psi)
        \\
        &=
        0.
        &(\because \cref{eq:tar_parametrization}, A_\theta\le 0)
    \end{align*}
    Summing them up and plugging \cref{eq:PLRL_marginal}, we have
    \begin{align*}
        \Lcal(V_\theta)
        &=
        \mathop{\doubleE}_{s\sim P_0}\sbr{V_\theta(s)}+
        \frac{\lambda}2
        \mathop{\doubleE}_{\psi\sim P_\Psi}
        \sbr{\cbr{J_\Psi(\psi)-V_\theta(\psi)}_+^2}
    \end{align*}
    and thus, since $J_\Psi(\psi)=V^*(\psi)$,
    \begin{align*}
        \Lcal(V_\theta)-\norm{\{V_\theta-V^*\}_+}_{2,P_\Psi}^2
        &=
        \mathop{\doubleE}_{s\sim P_0}\sbr{V_\theta(s)}+
        \frac{\lambda}2
        \mathop{\doubleE}_{\psi\sim P_\Psi}
        \sbr{J_\Psi(\psi)-V_\theta(\psi)}^2
        \\
        &=
        \mathop{\doubleE}_{s\sim P_0}\sbr{V_\theta(s)}+
        \frac{\lambda}2
        \mathop{\doubleE}_{(\psi,y)\sim P_{\Psi,Y}}
        \sbr{y-V_\theta(\psi)}^2 - \frac{\lambda\sigma^2}{2}
        \\
        &=
        \Lcal_{\mathrm{TAR}}(V_\theta) - \frac{\lambda\sigma^2}{2}.
    \end{align*}
    This completes the proof.
\end{proof}

In other words, $\Lcal_{\mathrm{TAR}}(V_\theta)$
is an upper bound on $\Lcal(V_\theta)$ (ignoring the constant shift $\lambda\sigma^2/2$)
that is tight around $V^*$, meaning it is a good surrogate loss for estimating $V^*$.
As a corollary, we can show that
the optimal parameter sets of penalized V-LP and TAR,
\begin{align}
    \Theta^*
    &\coloneqq \argmin_{\theta:V_\theta\ge 0}\Lcal(V_\theta), &
    \Theta^*_{\mathrm{TAR}}
    &\coloneqq \argmin_{\theta:V_\theta\ge 0} \Lcal_{\mathrm{TAR}}(V_\theta)
\end{align}
are the same, given the parametrization of $V_\theta$ is expressive enough.

\begin{corollary}
    Suppose there exists a parameter $\theta^*$ such that $V_{\theta^*}\in\argmin_{V\ge 0}\Lcal(V)$.
    Then, we have $\Theta^*=\Theta^*_{\mathrm{TAR}}$.
\end{corollary}
\begin{proof}
    By the assumption, we have $\Lcal(V_{\theta'})\le \Lcal(V')$ for all $\theta'\in\Theta^*$ and $V'\ge 0$.
    Moreover, we also have $V_{\theta'}\le V^*$ because otherwise the earlier assertion is violated as $\Lcal(\Vtil)<\Lcal(V_{\theta'})$
    with $\Vtil\coloneqq V_{\theta'}-\{V_{\theta'}-V^*\}_+ (\ge 0)$.
    Thus,
    for any $\theta$ such that $V_\theta\ge 0$, we have
    \begin{align}
        \Lcal_{\mathrm{TAR}}(V_{\theta'})
        &= \frac{\lambda \sigma^2}{2} + \Lcal(V_{\theta'})
        &\because \cref{eq:tar_loss_rewritten}\text{ and }V_{\theta'}\le V^*
        \\
        &\le \frac{\lambda \sigma^2}{2} + \Lcal(V_\theta)
        &\because V_{\theta'}\in\argmin_{V\ge 0}\Lcal(V)
        \\
        &\le \Lcal_{\mathrm{TAR}}(V_\theta),
        &\because \cref{eq:tar_loss_rewritten}
    \end{align}
    implying 
    $\theta'\in \Theta_{\mathrm{TAR}}$
    and hence $\Theta^*\subset \Theta^*_{\mathrm{TAR}}$.
    Finally, to prove the inverse,
    pick any $\theta''\in \Theta^*_{\mathrm{TAR}}$
    and observe that
    \begin{align}
        \Lcal(V_{\theta''})
        &\le
        \Lcal_{\mathrm{TAR}}(V_{\theta''})-\frac{\lambda \sigma^2}{2} &\because \cref{eq:tar_loss_rewritten}
        \\
        &\le 
        \Lcal_{\mathrm{TAR}}(V_{\theta^*})-\frac{\lambda \sigma^2}{2} & \because \theta''\in \Theta^*_{\mathrm{TAR}}
        \\
        &=
        \Lcal(V_{\theta^*}). &\because V_{\theta^*}\le V^*
    \end{align}
    This implies $\theta''\in \Theta^*$
    and hence $\Theta^*_{\mathrm{TAR}}\subset \Theta^*$.
\end{proof}

\section{Related Work}
\label{sec:related_work}
Applications of path regression and planning can be found abundantly in the literature of traffic modeling and route optimization~\cite{nikolova2008route,chen2024deep}.
In this context, one of the key aspects differentiating our method
is that the planning problem is embedded in the regression problem through the decomposition~(\cref{eq:advantage_decomposition,eq:tar_parametrization}),
whereas previously these two problems are handled separately or unified as an RL problem.
To illustrate this, let us compare our method to a conventional path regression methods, RETRACE~\citep{ide2011trajectory}.
Both TAR and RETRACE solve regression problems and decompose their predictions into individual contributions of path elements,
history-dependent terms $A_\theta(a^{k-1}, a_k)$ in TAR and 
edge-dependent terms $\Atil_\theta(a_{k-1}, a_k)$ in RETRACE.
However,
TAR's decomposition is readily usable for route planning (by maximizing $A_\theta(a^{k-1}, a_k)$ with respect to $a_k$ sequentially)
thanks to its interpretation as the advantage function,
whereas one needs to solve an additional planning problem to utilize RETRACE's decomposition
for finding the optimal routes.

\section{Conclusion}
\label{sec:conclusion}
We have studied the problem of offline path learning (PL)
through the lens of offline reinforcement learning (RL).
In particular, we have seen that
the connection between offline PL and offline RL
leads to a specific parametrization of the optimal value function.
The proposed parametrization is useful for interpreting the result of PL, decomposing the predicted yield into the contributions of individual elements in the path,
while avoiding the computational cost of solving full-fledged RL.

Practical implementation of TAR as well as its empirical evaluation
is left open as one of the most important future work.
Besides, it is interesting to see if regression-based formulations like TAR can be derived on top of other RL formalism other than penalized V-LP.

\bibliographystyle{alpha}
\bibliography{reference}

\end{document}